\documentclass[a4paper, 11pt]{article}
\usepackage[normalem]{ulem}
\usepackage{fullpage} 
\usepackage{microtype}
\usepackage{graphicx}
\usepackage{subfigure}
\usepackage{booktabs} % for professional tables
\usepackage{amsmath,amssymb,amsthm}
\usepackage[justification=justified,singlelinecheck=false]{caption}

\usepackage{hyperref}

\usepackage[round]{natbib}
\usepackage[utf8]{inputenc}
\usepackage[english]{babel}

\usepackage{soul}
\usepackage{xcolor}

\makeatletter
\newtheorem*{rep@theorem}{\rep@title}
\newcommand{\newreptheorem}[2]{%
\newenvironment{rep#1}[1]{%
 \def\rep@title{#2 \ref{##1}}%
 \begin{rep@theorem}}%
 {\end{rep@theorem}}}
\makeatother
\usepackage{float}
\newtheorem{theorem}{Theorem}
\newreptheorem{theorem}{Theorem}
\newtheorem{lemma}{Lemma}
\newreptheorem{lemma}{Lemma}

\usepackage{amsmath,amssymb,amsthm}

\title{Extreme Value Theory for Open Set Classification -- GPD and GEV Classifiers}
\author{Edoardo Vignotto$^1$ \and Sebastian Engelke$^1$}
\date{%
    $^1$Research Center for Statistics, University of Geneva, Geneva, Switzerland.%
}

\begin{document}

\maketitle

\begin{abstract}
Classification tasks usually assume that all possible classes are present during the training phase. This is restrictive if the algorithm is used over a long time and possibly encounters samples from unknown classes. It is therefore fundamental to develop algorithms able to distinguish between known and unknown new data. In the last few years, extreme value theory has become an important tool in multivariate statistics and machine learning. The recently introduced extreme value machine, a classifier motivated by extreme value theory, addresses this problem and achieves competitive performance in specific cases. We show that this algorithm can fail when the geometries of known and unknown classes differ, even if the recognition task is fairly simple. To overcome these limitations, we propose two new algorithms for open set classification relying on approximations from extreme value theory that are more robust in such cases. We exploit the intuition that test points that are extremely far from the training classes are more likely to be unknown objects. We derive asymptotic results motivated by univariate extreme value theory that make this intuition precise. We show the effectiveness of our classifiers in simulations and on real data sets.

\end{abstract}
\section{Introduction}

Modern classifiers achieve human or super-human performance in a variety of tasks \citep{christopher2016pattern}, including speech \citep{graves2013speech} and image recognition \citep{he2016deep}, but they are typically not able to discriminate between "known" and "unknown" classes and may give high confidence predictions for unrecognizable objects \citep{nguyen2015deep}. Here and throughout, we call a known class a class for which we have examples during the training phase, whereas an unknown class is a class for which we have no examples during that phase. The ability to distinguish between these two cases is important if there is the possibility that new classes arise in the future or if there have not been any examples of some classes in the training set due to their rarity. Recognizing unknown objects is the central goal of novelty and outlier detection. While detection algorithms can perform well in this task, they are usually not efficient to update when new training data arise and this is a serious limitation for an algorithm that is continuously used over a long time, also referred to as life-long learning.
The best way to address the problem posed by life-long learning is to build algorithms that have efficient updating mechanisms and hence support incremental learning, i.e., the ability to learn from the arrival of new training data. Algorithms of this type can however be computationally expensive, or rely only on heuristic principles, and they usually perform only closed set classification and do not recognize unknown objects.

The task of distinguishing between known and unknown classes in an incremental way is called open set classification or open world recognition \citep{bendale2015towards} and it is the natural combination of novelty detection and life-long learning. We underline that in this context standard hyper parameter optimization procedures such as cross-validation are usually not available, since in the training set there are only known objects. For this reason, an algorithm designed for open set recognition should involve as few hyper parameters as possible.

The principle goals in the context of open set classification are to build an algorithm that
\begin{itemize}
\item is capable of distinguishing between known and unknown objects;
\item supports incremental learning, i.e., an algorithm that is fast to update with the arrival of new training observations or classes;
\item does not rely on too many hyper parameters whose optimization is difficult.
\end{itemize}
The problem of distinguishing known and unknown objects is strongly related to the problem of estimating the support of the distribution of the known examples \citep{scholkopf2000support}. One direct way to perform open set recognition is to model the probability of an object of being from a known class \citep{scheirer2014probability} given its features. We follow the same route for the two algorithms that we propose.

The extreme value machine (EVM) introduced in \citet{rudd2018extreme} is an algorithm that uses extreme value theory to attack these problems. In the last few years, extreme value theory has become an important tool in multivariate statistics and machine learning. This is due to the fact that the extreme features, rather than the average ones, are the most important for discriminating between different objects \citep{scheirer2017extreme}. The EVM strongly relies on the geometry expressed by the known classes and hence can fail to distinguish between known and unknown objects when this geometry conveys misleading information about the unknown classes. For this reason, we propose two alternative approaches that use extreme value theory without using the geometry expressed by the known classes. We call them the GPD classifier (GPDC) and GEV classifier (GEVC), according to the respective approximations from extreme value theory that they rely on.
We underline that, in the following, we consider only the Euclidean distance for simplicity, but the obtained results are more general and can be applied to any distance.

This paper is organized as follows. In Section \ref{sec_relwork} we summarize the previous work in this context. Section \ref{sec_evt} states some fundamental results from extreme value theory that will be useful in the definition of the GEVC and GPDC. In Section \ref{sec_evm} we describe the EVM and explain the main limitations of this approach. In Section \ref{sec_gpdc} and \ref{sec_gevc} we describe the GPDC and GEVC and we theoretically justify their main properties. Finally, in Section \ref{sec_app}, we evaluate and compare the performance of these open set classification algorithms on both simulated and real data.

\section{Related Work}
\label{sec_relwork}
Novelty detection \citep{pimentel2014review} is a well established field that shares similar though not identical goals with outlier detection \citep{walfish2006review}. While the latter is focused on discovering which examples of the training data are not in agreement with the process that has generated the bulk of them, in novelty detection we suppose to have observations from a number of known classes and for a new sample we would like to determine whether it belongs to one of those existing classes or a new one. In this sense, novelty detection is more naturally related to classification. Machine learning techniques have been applied successfully to both tasks, improving considerably their performance \citep{abe2006outlier,shon2007hybrid,desir2013one}. 
Many novelty detection methods first estimate the distribution of the known data and then mark as unknown the new points associated with a low density regions \citep{bishop1994novelty}. We follow a similar approach for the two techniques that we propose. A relevant work in this direction is \citep{roberts1999novelty} who fit a Gaussian mixture model to the training data and use extreme value theory to decide if a new point is known or unknown. Our methods will not rely on parametric assumptions for the multivariate density, which makes them more widely applicable.

Extreme value theory has been recognized in the last years as a powerful tool to increase the performance of open set classification \citep{geng2018recent} and, more generally, machine learning techniques \citep{jalalzai2018binary} with particular interest to computer vision \citep{scheirer2017extreme}. The main reason for this success is the fact that the tails of the distribution of distances between training observations can be modeled effectively using the asymptotic theory provided by extreme value theory \citep{scheirer2011meta}. One example of this can be found in \citet{fragoso2013evsac} where the well-known RANSAC algorithm is improved using extreme value theory.

For incremental learning where the challenge is to build an algorithm that can easily be updated with the arrival of new training data, a simple approach is to use the {$k$-nearest neighbor classifier} \citep{james2013introduction}, a route that is to some extent also followed by the GPDC and the GEVC. Another method that has the desired requirement to be naturally adapted to new data is the nearest class mean classifier \citep{mensink2012metric}, which represents each class as a prototype vector that is the mean of all the examples for that class seen so far. More complex methods are based on machine learning approaches such as support vector machines \citep{ruping2001incremental}, neural networks \citep{rebuffi2017icarl} or random forests \citep{saffari2009line} and are in general computationally more demanding.

Finally, we note that the formulation of GPDC is partially related with the estimation of the right end point of a univariate distribution. As we will argue in the following, in the GPDC this problem takes a simpler form than in the general problem of upper end point estimation in extreme value statistics \citep{hall1982estimating}. 
\section{Extreme Value Theory}
\label{sec_evt}

\begin{theorem}[Fisher--Tippett--Gnedenko theorem \citep{coles2001introduction}]\label{thm1}
Let $X_1,X_2,\dots$ be a sequence of i.i.d.~samples from the distribution function $F$. Let $M_n=\max(X_1,\dots,X_n)$. If there exist sequences $a_n\in\mathbb R$, $b_n > 0$ such that
$$ P\left(\frac{M_n-a_n}{b_n} \le z\right) \rightarrow  G(z), \qquad n\rightarrow \infty,$$
then if $G$ is a non-degenerate distribution function it belongs to the class of generalized extreme value (GEV) distributions with 
$$ G(z)=\exp \left[ -\left\{ 1+\xi\left(\frac{z-\mu}{\sigma}\right)\right\}^{-1/\xi} \right], \quad  z\in\mathbb R: 1+\xi\left(\frac{z-\mu}{\sigma}\right)>0, $$
where $\xi \in \mathbb R$, $\mu \in \mathbb R$ and $\sigma>0$ are the shape, location and scale parameters, respectively.
\end{theorem}
A similar theorem can be formulated for threshold exceedances.
\begin{theorem}[\citealp{coles2001introduction}]\label{thm2}
Assume that Theorem 1 holds. Then, as $u$ tends to the upper endpoint of $X$, the distribution function of $X-u$, conditional on $X>u$, is approximately
$$H(y)=1-\left( 1+\frac{\xi y}{\bar\sigma} \right)^{-1/\xi}, \quad y>0: 1+\frac{\xi y}{\bar\sigma}>0,$$
where $\bar\sigma=\sigma+\xi(u-\mu).$ The distribution $H$ is called the generalized Pareto distribution (GPD).
\end{theorem}

If we are interested in computing the distribution of the maximum of a large number of independent copies a given random variable, the first theorem suggests to fit a GEV distribution to various such maxima taken over blocks of the same lengths. On the other hand, if we are interested to model the right tail of a distribution, then the second theorem suggests to fit a GPD to all values above a high threshold of a sequence of independent replications of a random variable with this distribution. The two approaches are asymptotically equivalent and the parameters in the two distributions are deterministically linked. 
The following lemma will be useful to describe our algorithms. 

\begin{lemma}[\citealp{coles2001introduction}]\label{lem3}
Assume that Theorem 1 holds with $\xi < 0$ and let the upper endpoint of $F$ be denoted by $b^*$. Then, 
$$ P\left(\frac{M_n-a_n}{b_n} \le z\right) \rightarrow  W(z), \qquad n\rightarrow \infty,$$
where
$$W(z)=\begin{cases} \exp \left\{ -\left(  \dfrac{b^*-z}{\sigma}  \right)^{\alpha} \right\}, & \mbox{if }z<b^*, \\ 1, & \mbox{if }z\ge b^*,
\end{cases}$$
where $\alpha=-1/{\xi}>0$.
The distribution $W$ belongs to the reversed Weibull family and it is a special case of the GEV distribution. Moreover, we have the representation $b^*=u-\bar\sigma/\xi$ in terms of the parameters in Theorem \ref{thm2}.
\end{lemma}

\section{The Extreme Value Machine} 
\label{sec_evm}

In this section firstly we briefly describe the extreme value machine introduced in \citet{rudd2018extreme} and then we underline its possible limitations. 

\subsection{Algorithm description}

The idea of the EVM is to approximate the distribution of the margin distance of each point in each class using extreme value theory. A new point is then classified as known if it is inside the margin of some point in the training set with high probability.

Denote the training data by $x_i \in \mathbb{R}^p$, each of them labeled as a class $y_i\in\{C_1,\dots, C_J\}$, $i=1,\dots,n$.
Here and throughout, $p\in\mathbb N$ is the dimension of the predictor space and $J\in\mathbb N$ the number of different classes in the training set. Further, suppose that we have a new unlabeled point $x_{0}$ that we would like to mark as known or unknown. In other words, our goal is to decide if $x_{0}$ comes from the distribution of the training data or not.  We introduce the concept of margin distance of a point $x_i$ as half of the minimum distance between $x_i$ and all the points belonging to a different class in the training data set. More formally, the margin distance of $x_i$ is 
$$M^{(i)}=\min_{j:y_j\ne y_i } D_{j}^{(i)}=\min_{j:y_j\ne y_i } \frac{||x_i-x_j||}{2}.$$
The idea is to model the lower tail of the distribution of $M^{(i)}$ by using the $D_{j}^{(i)}$ as the input data. An equivalent problem, to which we can apply extreme value theory, is to retrieve the upper tail of the distribution of $\bar M^{(i)} = \max_{j:y_j\ne y_i } -D_{j}^{(i)}$. To this end, the authors propose to use the Fisher--Tippett--Gnedenko theorem (cf., Theorem \ref{thm1}) to fit a GEV distribution to the $k$ largest observed $-D_{j}^{(i)}$ for each point $x_i$. They claim that the asymptotic distribution of $\bar M^{(i)}$ is 
\begin{equation} W^{(i)}(z)=\begin{cases} \exp \left\{ -\left( - \dfrac{z}{\sigma_i}  \right)^{\alpha_i}\right\}, & \mbox{if }z<0, \\ 1, & \mbox{if }z\ge 0,
\end{cases}
\label{b}
\end{equation}
assuming $b^*=0$ since $\bar M^{(i)}$ is bounded above by zero as a negated distance. 
%%Note that, if we denote with $m_i^0$ the minimum distance between $x_0$ and any object belonging to the class $C_i$, i.e, $m_i^0=\min_{j:y_j = C_i } ||x_0-x_j||$,  then $\Psi_i(-m_i^0)$ represents how likely it is that $x_{0}$ is inside the margin of any point in that class and hence belongs to it.
%%We can now formally define the probability that a new object $x_{0}$ falls inside the margin of the training point $x_i$, and then that $x_{0}$ is known given $x_i$, as
%%$$P(\text{known}|x_i)=\Psi_i(-||x_0-x_i||).$$
%%We classify $x_{0}$ as known if there is at least one training point for whom this probability is higher that a given threshold $\alpha$,
%%$$x_{0} \text{ is known if } \text{argmax}_{i=1,\dots,n} P(\text{known}|x_i) \ge \alpha.$$
Estimating the parameters $\alpha_i$ and $\sigma_i$ for each training observation $x_i$ separately yields estimates $\hat W^{(i)}$ of the distribution functions $W^{(i)}$. Those estimates are used to label a new point $x_0$ as known if it is likely enough to fall in at least one of the margins of the $n$ training data, that is, if
\begin{align}\label{eq2}
	\max_{i=1,\dots,n} \hat W^{(i)}\left(-||x_0-x_i||\right) \geq \delta,
\end{align}
and it is labeled as unknown otherwise. Here, $\delta$ is a probability threshold that is chosen by a heuristic formula in \citet{rudd2018extreme}.

\subsection{Limitation of the EVM}
The EVM redefines the open set classification problem as a probabilistic framework within extreme value theory, but it has some drawbacks that we describe here and that will be the starting point to construct new methods that improve upon it.

First, assuming in Equation \ref{b} that the upper endpoint is $b^*=0$ implies that all classes in the training set share the same support and that it is impossible that two classes are perfectly separated. Indeed, if two classes are perfectly separated, then the distance between two points sampled from them is bounded from above by a strictly positive constant. This assumption is rather restrictive and it is hardly verifiable on observed data. Second, estimating the upper tail of $\bar M^{(i)}$ based on the $k$ largest observed $-D_{j}^{(i)}$ for each $x_i$ should rely on Theorem \ref{thm2} instead of Theorem \ref{thm1} since this amounts to a threshold exceedance approach with a GPD as limiting distribution.

A further limitation of the EVM is related to the choice of the threshold $\delta$ in \eqref{eq2}. This threshold controls the open set classification error and a bigger $\delta$ implies a higher probability to label a new object as unknown. In a statistical setting this type of thresholds are chosen by fixing the type I error, 
in this case, the probability to classify a known object as unknown. Unfortunately, for the EVM algorithm the authors do not propose any procedure to choose $\lambda$ controlling the type I error and thus it is not clear how to properly tune it.

The most important drawback of the EVM is that it gives a non-justified premium to known classes far from all the others. In fact, if there are a lot of classes relatively close to each other and one other class far from them, then all the points of the latter class will have a large margin distance. For this reason, if we observe a new point from an unknown class located closer to this far away class than the other known classes, we always classify the new point as belonging to the far away class.
The main issue here is that we cannot use the geometry given by the known classes to do open set classification because it does not convey any information regarding the unknown ones. Knowing that class $A$ is very far from class $B$, does not necessarily imply that a new object $x_{0}$ is known only because the distance between it and the closest point in class $A$ is much smaller than the distances between classes $A$ and $B$. It might simply be that the new class is closer to class $A$ than is class $B$. The EVM is implicitly assuming that the geometry of the known classes reflects also the behavior of the unknowns, a strong assumption that is again hard to verify. We will illustrate this problem in a simulation study in Section~\ref{sec_app}.
\section{The GPD Classifier}
\label{sec_gpdc}
\subsection{General setting}
In this section we propose an algorithm for open set classification that uses the principles of extreme value theory and that does not rely on the geometry of the training set. This algorithm, which we call the GPDC, considers the distances between the new point that we want to mark as known or unknown, and the points composing the training data set. For simplicity, and since we are primarily interested in open set classification, all the known classes are seen as one big class. If there is high confidence that the lower endpoint of the distribution of these distances is zero, then the process that has generated the training data can also have generated the new point, and we thus mark the new point as known; otherwise we mark it as unknown. 

To formally describe the algorithm, we first show some basic results that will help in understanding the GPDC. As before, we assume to have training data $x_i\in\mathbb R^p$, $i=1,\dots,n$, with class labels $y_i\in\{C_1,\dots, C_J\}$. In total we have $J$ known classes and we assume that each class is described by a continuous density function $f_{C_j}$ defined on $\mathbb R^p$, $j=1,\dots,J$, where the probability that a point in class $C_j$ falls in the set $A\subset \mathbb R^p$ is
$$\int_{x\in A} f_{C_j}(x)\, dx.  $$
The process that has generated the training data set can be described as a mixture of these density functions, with unconditional density
$$f(x)=\sum_{j=1}^J w_j f_{C_j}(x), \quad x\in\mathbb R^p,$$ 
for weights $w_j\in[0,1]$ with $\sum_{j=1}^J w_j =1$. In this sense, the value of the function $f$ is large evaluated at a point that has high chance to be known. Note that, in the following, we approximate directly $f$ and then we do not have to estimate the weights $w_j$. 

\subsection{Extreme value theory and open set classification}

In this setting the training data $x_i\in \mathbb{R}^p$, $i=1,\dots,n$, can be thought as being sampled from only one class, the class of the knowns, with density $f$. Given a new point $x_{0}$ without label, we are interested in deciding whether it is from a known class and has been generated by $f$, or whether it belongs to an unseen class and has been generated by an unknown density~$f_0$. 

We denote with $D_i$, $i=1,\dots,n$, the distances $||x_i-x_{0}||$ between $x_{0}$ and the points in the training set. In the following, we are interested in the lower tail of the distribution $D$ of the distance between $x_{0}$ and a generic point sampled from $f$. Intuitively, if the lower tail of $D$ behaves similarly as the corresponding quantity for training samples, then $x_{0}$ has a high probability to be a known point. 
More mathematically, let $B_{\delta}(x_0)$ denote a ball of radius $\delta$ centered in $x_{0}$. Then $$P\{B_{\delta}(x_0)\}=P(D<\delta)=P(-D>-\delta).$$
Moreover, the upper end point of $-D$ is equal to zero if and only if there is no $\delta>0$ such that $P\{B_{\delta}(x_0)\}=0$. Only in this case there is a positive probability that $x_{0}$ comes from the process that has generated the training data, i.e., that it is a known point.  

Note that $f$ is, in general, a density on a multivariate space, whereas $P(-D>\cdot)$ is a simple univariate survival function. We can therefore model the right tail of $-D$ using extreme value theory. The following results make this more explicit.

\begin{theorem}\label{llh_thm}
Denote the distances between $x_{0}$ and the points in the training set with $D_1,\dots,D_n$ and assume that $x_{0}\in \text{supp}(f)$ is in the support of the known classes, that is, the upper end point of $-D$ is zero. Suppose the conditions of Theorem \ref{thm2} hold, then the distribution of $R_i - u$, where $R_i =-D_i$, above a high threshold $u<0$ (close to $0$) can be approximated by a GPD distribution $\widehat H$ with log-likelihood
\begin{equation}\label{eq1}
\log L(R_1,\dots,R_n; \xi) \propto -k\log\xi-\frac{1}{\xi} \sum_{i=1}^{k} \log\left(\frac{R_{(n+1-i)}}{u}\right),
\end{equation} 
where $R_{(n)} \geq R_{(n-1)} \dots \geq R_{(1)}$ are the order statistics of the $R_i$ and $k$ is the number of exceedances above~$u$.

The maximum likelihood estimator of $\xi$ is then
\begin{align}\label{hill_stat}
	\widehat\xi_n=\frac{1}{k}\sum_{i=1}^{k} \log \left(\frac{R_{(n+1-i)}}{u} \right).
\end{align}

\end{theorem}

\begin{proof}
Obviously, if $x_{0}$ comes from one of the classes composing the training set, the upper end point of $-D$ is zero since the minimum possible distance between $x_{0}$ and a point generated from $f$ is zero. 
We can use Theorem 2 to approximate the exceedances of $R_1,\dots, R_n$ over a high threshold $u$ using a GPD. The log-likelihood of the exceedances is
\begin{align*}
\log &L(R_1,\dots,R_n;\bar\sigma,\xi) \\
&\propto -k\log\bar\sigma-\left(1+\frac{1}{\xi}\right) \sum_{i=1}^{k}\log \left(1 + \frac{\xi (R_{(n+1-i)} - u)}{\bar\sigma}\right). 
\end{align*} 
Using Lemma 1 and the fact that the upper endpoint of $-D$ is zero, we have that $u-\bar\sigma/\xi=0$ and thus $\bar\sigma=u\xi$. Plugging this in, we obtain the simplified likelihood in \ref{eq1} up to additive constants, parametrized only by $\xi$. The MLE of $\xi$ is then
\begin{align*}
\widehat\xi_n & =\text{argmax}_{\xi} \log L(R_1,\dots,R_n,\xi) \\
&=\text{argmax}_{\xi}  -k\log\xi-\frac{1}{\xi} \sum_{i=1}^{k} \log\left(\frac{R_{(n+1-i)}}{u}\right) \\
&=\frac{1}{k} \sum_{i=1}^{k} \log\left(\frac{R_{(n+1-i)}}{u}\right).
\end{align*}
\end{proof}

The above result motivates to use the statistic \eqref{hill_stat} in an hypothesis test to decide whether
the new point $x_0$ can come from a known class. The following results provides the theoretical background for this test.

\begin{theorem}\label{thm5}
 Assume the same conditions as in Theorem \ref{llh_thm}.  Choose the threshold $u$ to be the order statistic $R_{(n-k)}$, and assume that $k = k(n)\to \infty$ and $k(n)/n \to 0$, as $n\to\infty$. The shape parameter of the distribution of $-D$ is $\xi = -1/p$ and the maximum likelihood estimator in \eqref{hill_stat} converges in probability 
 $$ \widehat \xi_n \stackrel{P}{\to} -1/p,$$
 where $p \in \mathbb N$ is the dimension of the predictor space.
\end{theorem}

\begin{proof}

We first note that the survival function of $-D$ behaves 
at zero as
\begin{align*}
	P(-D> -\delta) &= P\{B_{\delta}(x_0)\}\\
    &=\int_{B_{\delta}(x_0)} f(x) \, dx\\
    &= f(x_{0}) V\{B_{\delta}(x_0)\} + o(1)\\
    &= f(x_{0}) \frac{\pi^{p/2}}{\Gamma(p/2+1)} \delta^p  + o(1),
\end{align*}
where $V\{B_{\delta}(x_0)\}$ is the volume of the ball $B_{\delta}(x_0)$ in $\mathbb R^p$, and the approximation holds as $\delta \to 0$.
Consequently, the distribution of $-D$ is in the max-domain of a GEV distribution with shape parameter $\xi = -1/p$.
This implies that $1 / D$ has a regularly varying survival function
$$P(1 / D > x) = \ell(x) x^{-p},$$
for some slowly varying function $\ell$ at infinity.
The upper tail of $1/D$ is thus in the max-domain of attraction of a GEV distribution with shape parameter $\tilde \xi = 1/p>0$. Rewriting the statistic $\widehat\xi$ using $u= R_{(n-k)}$ gives
\begin{align*}
	\widehat\xi_n= - \frac{1}{k}\sum_{i=1}^{k} \log \left(\frac{-1 / R_{(n+1-i)}}{-1 / R_{(n-k)}} \right),
\end{align*}
and we recognize the classical Hill estimator \citep{de2007extreme} applied to the sequence of independent copies $-1/R_1,\dots, -1/R_n$ of $1/D$.
The consistency of the Hill estimator (cf., Theorem 3.2.2 in \citet{de2007extreme}) and the fact that $1/D$ has shape parameter $1/p$ yields the desired result.
\end{proof}

So far we have supposed that the upper end point of $-D$ is zero. Under this null hypothesis, i.e., that the new point is known, the above theorem shows that the statistic $\widehat \xi$ is approximately the negative reciprocal of the predictor space dimension. Moreover, it is possible to show that the estimator \eqref{hill_stat} is a special case of the classical Hill estimator \citep{de2007extreme} and hence it is asymptotically normal under a second order condition. Using this fact, we could even derive asymptotic confidence intervals for $\xi$ based on it. The assumption on $k(n)$ is common in extreme value statistics to ensure the right trade-off between bias and variance of a tail estimator. 

Under the alternative, namely that the new point $x_0$ is unknown, we can encounter two situations. First, if the support of the new class is overlapping with the support of the known classes $\text{supp}(f)$, then $-D$ might have upper endpoint zero if $x_0$ falls into $\text{supp}(f)$.
Second, if the supports are not overlapping or $x_0\notin\text{supp}(f)$, then the upper endpoint of $-D$ is strictly smaller than zero. In this second case, the statistic $\widehat \xi$ can be shown to converge to zero.

\begin{theorem}\label{thm6}
Let $r^*$ be the upper endpoint of $-D$ and suppose that $r^* < 0$. Choose the threshold $u$ to be the order statistic $R_{(n-k)}$, and assume that $k = k(n)\to \infty$ and $k(n)/n \to 0$, as $n\to\infty$. Then the statistic in \eqref{hill_stat} converges almost surely to zero, that is,
 $ \widehat \xi_n \stackrel{a.s.}{\to} 0.$
\end{theorem}

\begin{proof}
	We note that 	
	\begin{align*}
	0 &\leq -\widehat\xi_n = \frac{1}{k}\sum_{i=1}^{k} \left\{\log(-R_{(n-k)}) - \log(-R_{(n+1-i)}) \right\}\\
    &\leq \log(-R_{(n-k)}) - \log(-r^*) .
\end{align*}
Since $-R_{(n-k)} \to -r^*$ almost surely as $n\to\infty$, this proves the assertion.
\end{proof}

We can use this fact to do a first test and mark a new point $x_0$ for which $p \widehat\xi_n$ is significantly larger than $-1$ as an unknown point.
Note that we use the quantity $p \widehat\xi_n$ instead of $ \widehat\xi_n$ to
stabilize the asymptotic variance and to avoid numerical instabilities as $-1/p\to 0$ as $p\to\infty$. In fact, under a second order condition and by asymptotic normality of the Hill estimator (cf., Theorem 3.2.5 in \citet{de2007extreme}), the quantity $p \widehat\xi_n$
is approximately normal with mean $-1$ and variance $1$.
With this step we have excluded all the points that have no possibility to belong to one of the classes of the training set.

If, on the other hand, $p \widehat\xi_n$ is close to $-1$, we cannot reject the null hypothesis that $x_0 \in \text{supp}(f)$. In this case, Theorem \ref{llh_thm} with threshold $u = R_{(n-k)}$ yields the tail approximation for $x > R_{(n-k)}$
\begin{align}\label{GPD_approx}
	P(-D > x) & \approx \frac{k}{n} \left\{1 - \widehat H(x-R_{(n-k)})\right\} = \frac{k}{n} \left(x/R_{(n-k)}\right)^{-1/\widehat \xi_n},
\end{align}
where $\widehat H$ is the fitted GPD distribution.
There is thus the possibility that $x_0$ was generated from one of the known classes. Obviously, we have to exclude points that have positive, but very small probability of being knowns.
We therefore compute the size of the ball around $x_0$ that contains a fixed amount $\gamma>0$ of mass of $f$. More precisely, we let $q_\gamma<0$ be the $(1-\gamma)$-quantile of $-D$ for some small $\gamma < k/n$, which we can compute from \eqref{GPD_approx} as 
$$q_\gamma = R_{(n-k)} + \widehat H^{-1}(1-n\gamma/k) =  R_{(n-k)} (n\gamma/k)^{-\widehat \xi}.$$
Then $\widehat P\{B_{-q_\gamma}(x_0)\} = 1- \gamma$, and the smaller $-q_\gamma$ the higher the training density $f(x_0)$ around $x_0$. Thus, if $-q_\gamma$ obtained for the new point $x_0$ is significantly larger than the corresponding quantity for a generic point in the training set, we mark $x_0$ as unknown since the density $f(x_0)$ around $x_0$ is too small. Otherwise we mark it as known. The probability $\gamma$ should be chosen small enough to have a good approximation of the magnitude of the density around $x_0$, but at the same time large enough to obtain a reliable estimate of $q_\gamma$. We choose $\gamma =1/n$ since this is sufficiently small and 
$ R_{(n-k)} + \widehat H^{-1}(1-1/k)$ is usually a very good approximation of the true quantile.

\subsection{The GPDC algorithm}\label{GPDCalgorithm}
Using the results on the asymptotic behavior of the statistic~$\widehat \xi_n$, we propose the following algorithm to classify a new point~$x_{0}$.
\begin{enumerate}
\item Compute the negated distances $-D_1,\dots,-D_n$ between $x_{0}$ and each point in the training set.
\item Estimate $\widehat\xi_n$ using only the biggest $k$ negated distances $R_{(n)}, \dots, R_{(n+1-k)}$.
\item Perform the hypothesis test
\begin{align*}
&H_0\text{: }p \xi_n=-1\\
&H_1\text{: }p \xi_n=0.
\end{align*}
If $p\widehat\xi_n$ is smaller than a given threshold $s>-1$, mark $x_{0}$ as possibly known and go to the next point, otherwise mark it as unknown and exit the algorithm.
\item Compute the $(1-1/n)$-quantile of $-D$ by $q_{\gamma}=R_{(n-k)} + \widehat H^{-1}(1-1/k)$ using the estimated GPD $\widehat H$ and $\gamma=1/n$.
\item If the radius $-q_\gamma$ is bigger than a given threshold $t>0 $ mark $x_{0}$ as unknown, otherwise mark it as known. This step can be viewed as a second hypothesis test.
\end{enumerate}
Some of the steps require further explanation. 
The number of upper order statistics $k$ in the second step that is used to obtain $\widehat\xi_n$ corresponds to the classical bias-variance trade-off of the Hill estimator. There are widely used diagnostics plots that help to determine the best $k$. The outcome of our GPDC algorithm does not strongly depend on the choice of $k$, as soon as it is chosen small enough and of the order $k = k_n = o(n)$, as suggested by Theorem \ref{thm6}. 

Step 3.~and step 5.~require the choice of the thresholds 
$s$ and $t$ to perform the hypothesis tests. It is common to fix the type I error to some probability $\alpha$ such as $5\%$, and then to compute the threshold that realizes this error. To this end we require the distribution of the test statistic under the null hypothesis. Instead of relying on asymptotic results, we propose to execute the algorithm using all the points in the training set as unknown one after the other in a jackknife fashion and to obtain in this way each time $\widehat\xi^{(i)}_n$ and $-q_{\gamma}^{(i)}$, $i=1,\dots,n$. We then jointly set the thresholds $s$ and $t$ to the $(1-\alpha/2) \%$ quantiles of $\widehat\xi^{(1)}_n,\dots,\widehat\xi^{(n)}_n$
and $-q_\gamma^{(1)},\dots,-q_\gamma^{(n)}$, respectively, such that we mark at most $\alpha \%$ of the training data as unknown following Bonferroni's correction for multiple testing \citep{shaffer1995multiple}. Note that this requires to compute $O(nk \log n )$ distances, roughly the same as in the training phase of the EVM. We can do this procedure at the beginning of the life of the algorithm and repeat it once in a while if new training data arise.

We emphasize that, contrary to what it is stated in the EVM paper \citep{rudd2018extreme}, one should be careful to choose hyper parameters based on cross-validation by randomly splitting the classes in the training set in knowns and unknowns. This can be misleading since known classes convey no information on the unknowns.

It is important to underline that this algorithm relies on the asymptotic approximation by a GPD and hence it is asymptotically exact as the number of training samples tends to infinity. Moreover, it is completely kernel free and its implementation is very efficient using fast nearest-neighbor searching \citep{arya1998optimal}. In fact, during the evaluation of a new point $x_{0}$ we have to compute only $O(k\log n)$ distances, whereas with the EVM we have to compute $O(n)$ distances. We stress also that the algorithm supports incremental learning since it does not require any training procedure once the thresholds for the hypothesis tests are fixed. If the interest is not only in open set classification our algorithm can be completed with an incremental classifier that performs standard classification only on the points marked as knowns.

Considering all training data as belonging to only one class is not a limitation of the algorithm. In fact, if one distribution $f_{C_j}$ dominates the others (especially regarding the tail behaviour) and thus causes a mixture $f$ that is not informative for all training classes, a simple solution is to train separate GPDCs for each class and marking a new point as unknown if all the classifiers mark it as unknown. The same consideration holds for the GEVC.

\section{The GEV Classifier} 
\label{sec_gevc}
The GPDC in the previous section was based on the approximation of threshold exceedances by the GPD distribution according to Theorem \ref{thm2}. In this section we propose another algorithm to do open set classification without relying on the geometry of the observed data. It is based on the approximation by the GEV distribution given in Theorem \ref{thm1}, and we thus refer to it as the GEVC.

For simplicity, as before, all known classes are collapsed into one big class and we denote the training data with $x_i$, $i=1,\dots,n$, and with $f$ the probability density describing the distribution of the training data. For each $i=1,\dots, n$, we compute the distance between the training point $x_i$ and the nearest training point to it, i.e.,
$$D_i^{min}=\min_{j\ne i}||x_i-x_j||. $$
This phase, using fast neighbor search, requires to compute only $O(n\log n)$ distances. Note that the quantities
$$-D_i^{min}=\max_{j\ne i}-||x_i-x_j|| $$
are bounded above by zero. Even though these random variables are not exactly independent, their dependence seems weak enough to still apply methods from extreme value theory. Using this and Lemma \ref{lem3} we can fit a reversed Weibull distribution to $-D_1^{min},\dots,-D_n^{min}$, and thus estimating the distribution of $-D^{min}$, i.e., the distribution of the maximum of the negated distances between a known point and all the remaining points in the training set. Denote this estimated reversed Weibull distribution function by $\widehat W$.

When a new point $x_0$ arises and we want to mark it as known or unknown we consider the statistic
$$ -d_0^{min} = \max_{i=1,\dots, n}-||x_i-x_0||,$$
in order to perform the following hypothesis test:
\begin{align*}
&H_0\text{: }x_0 \text{ is known}\\
&H_1\text{: }x_0 \text{ is unknown}.
\end{align*}
Note that this operation requires to compute only $O(\log n)$ distances.
Under the null hypothesis, $-d_0^{min}$ is approximately a sample of the distribution $-D^{min}$. Thus we expect the quantity 
$$P(-D^{min} < -d_0^{min})$$
to be not too small, whereas under the alternative hypothesis this is not guaranteed. We can then fix a level $\alpha>0$ for the type I error to determine if the new point $x_0$ is known or unknown. Given the estimated distribution $\widehat W$ that approximates $-D^{min}$, we reject the null hypothesis if
$\widehat W(-d_0^{min}) < \alpha$. In this case, the smallest distance of the new point to a training sample is too large, and we therefore mark it as unknown. Otherwise it is marked as known. Note that, also with the GEVC, the threshold $\alpha$ permits to control directly the type I error. Moreover, the GEVC has no hyper parameters and it is fast to update when we add new data to the training set, since revising $-D_i^{min}$ requires only to compute the minimum distance between $x_i$ and the added points.
\section{Application}
\label{sec_app}

\begin{figure*}[htp]
  \centering
  {\includegraphics[scale=0.43]{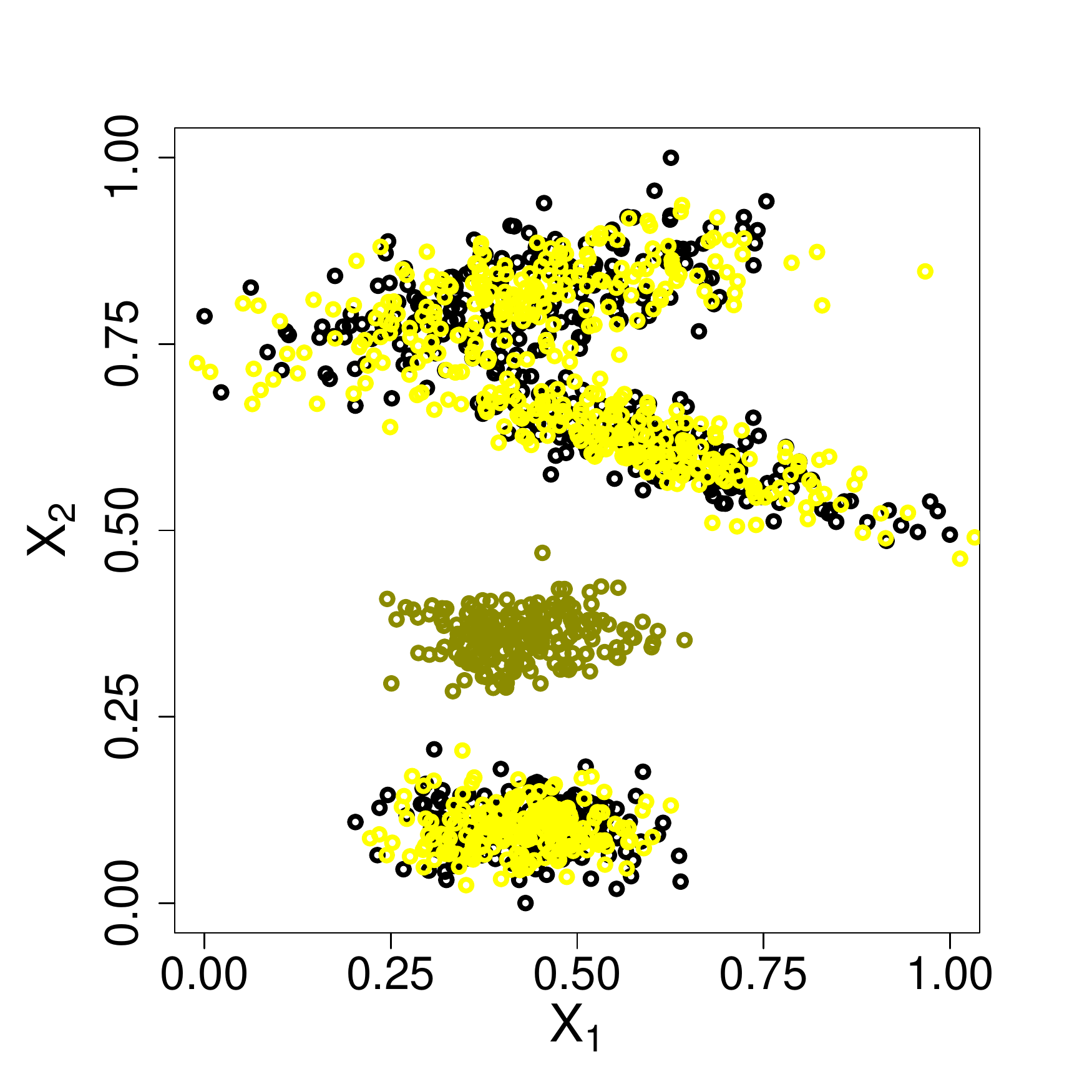}}\quad
  \subfigure{\includegraphics[scale=0.43]{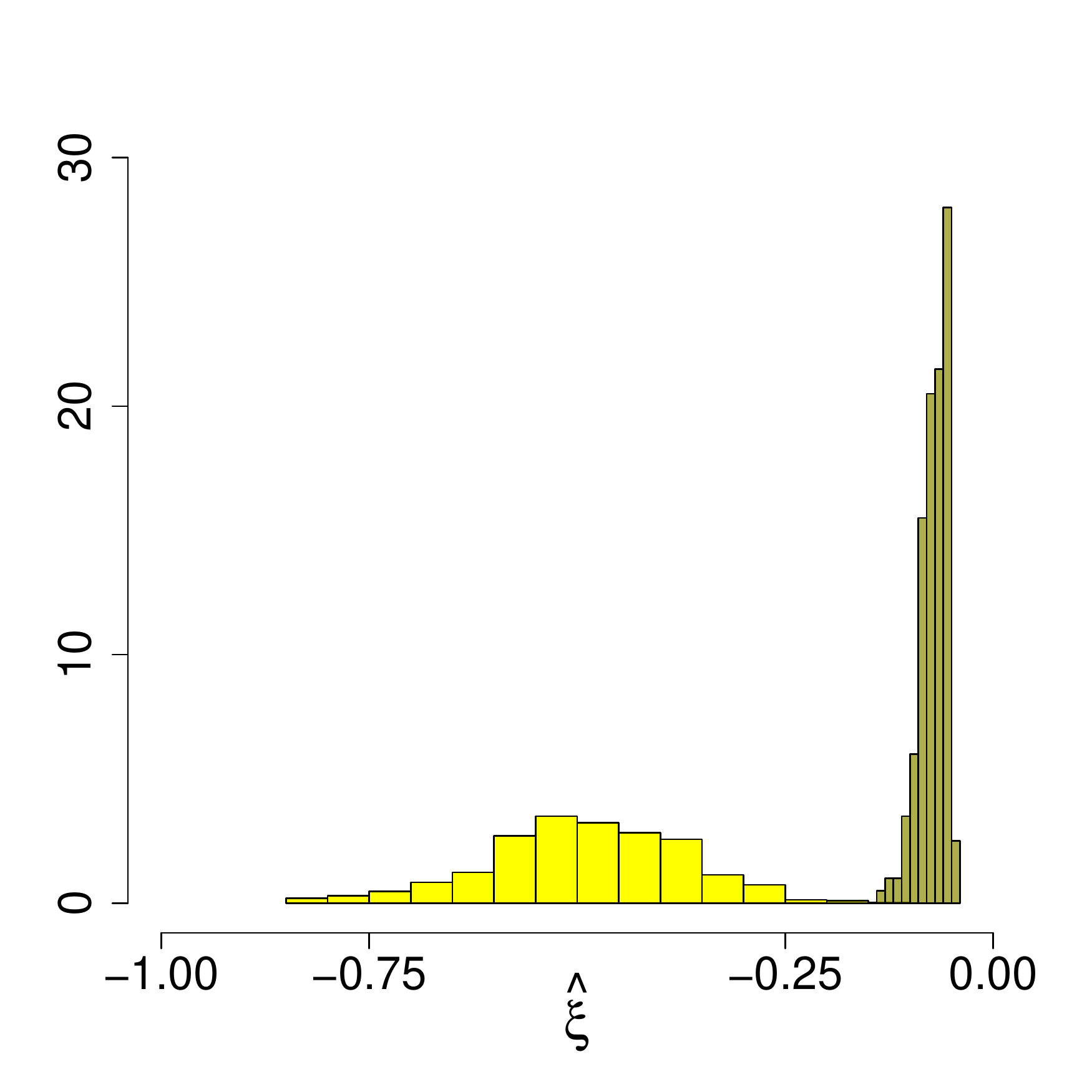}}
\caption{Left panel: the simulated dataset for the toy example: training data (in black), the known (bright yellow) and unknown examples (dark yellow) from the test set.
Right panel: the $\widehat\xi_n$ estimates for known (bright yellow) and unknown test data (dark yellow) for the simulated dataset.}
\label{fig:1}
\end{figure*}

\begin{figure*}[htp]
  \centering
  {\includegraphics[scale=0.43]{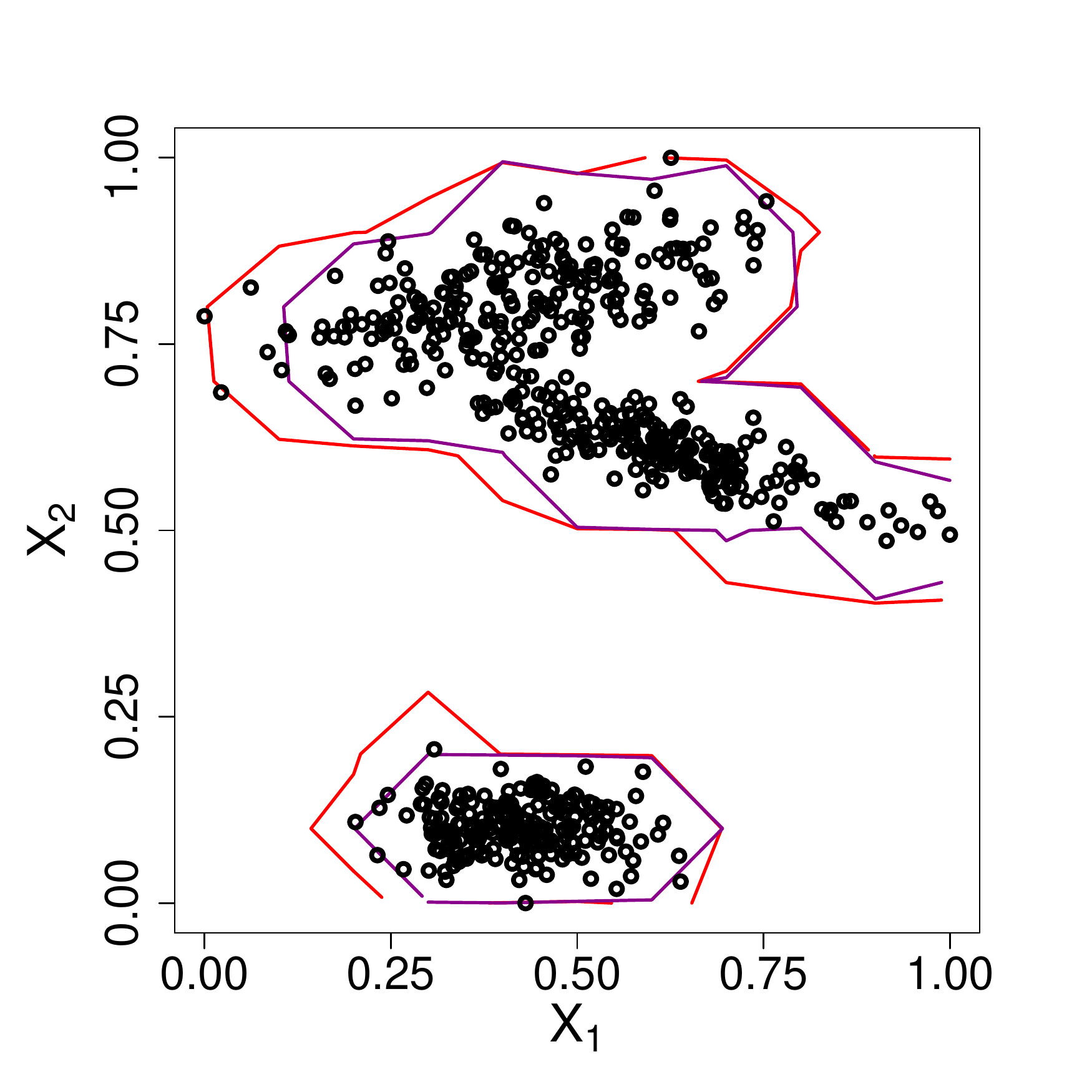}}\quad
  \subfigure{\includegraphics[scale=0.43]{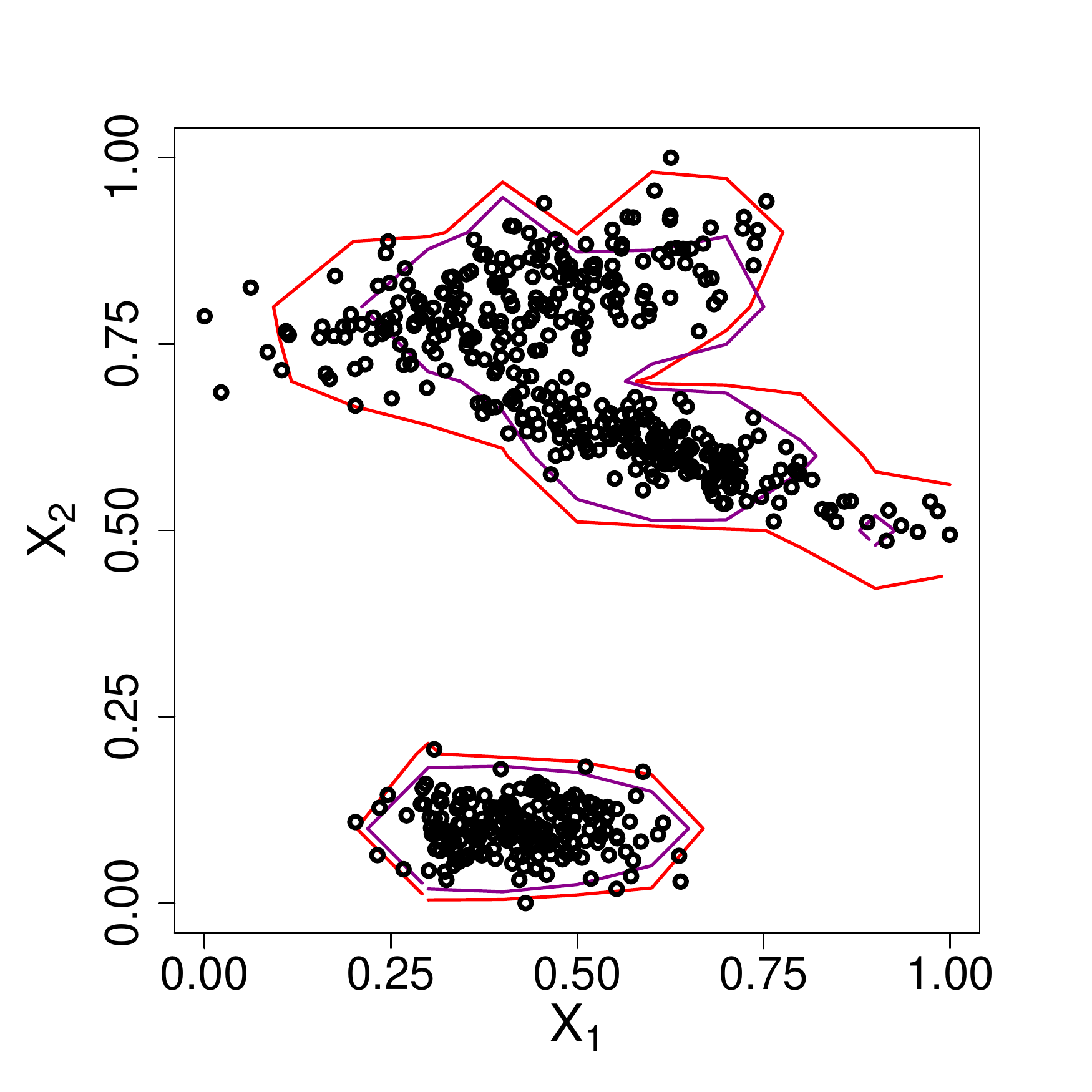}}
\caption{Left panel: decision boundaries for the GPDC (in red) and the GEVC (in magenta) with $\alpha=0.01$.
Right panel: decision boundaries for the GPDC (in red) and the GEVC (in magenta) with $\alpha=0.1$.}
\label{fig:contour}
\end{figure*}

In this section we compare the EVM introduced in \citet{rudd2018extreme} with our GPDC and the GEVC on simulated and real data. For completeness, when using real data, we compare the methods also with the One-Class SVM \citep{scholkopf2000support}, a state of the art method for novelty detection. For this last technique, we use radial kernel and we select its hyper parameters using part of the data as validation set, including both known and unknown objects, following what is proposed in the original paper. This approach is possible only in an experimental setting, since in real open world recognition we do not have any access to the unknown objects during the training phase.
\subsection{Simulated data}

We begin our experimental evaluation with a toy example that shows that the EVM may perform poorly when the geometry given by the training data set conveys misleading information about the unknown classes. The training set has $n=600$ observations and it is composed of three classes, each one of them is sampled from a different bivariate normal distribution, i.e., the dimension of the predictor space is \mbox{$p=2$}. The test data set has 800 observations and it is composed of examples from these three classes plus another unknown class, sampled from another bivariate normal distribution. Both the train and the test data are shown in the left panel of Figure \ref{fig:1}. The unknown objects are well separated from the known ones. 

We apply the EVM, the GPDC and the GEVC to this data set and, to solve the problem given by the fact that the probability threshold of the EVM is not interpretable as those of the GPDC and the GEVC, we choose to evaluate the performance of the different algorithms using the obtained area under the ROC Curve (AUC) \citep{bradley1997use} that does not require to fix a threshold. For both the EVM and the GPDC we use $k=20$, but the results are stable for different values of $k$. 
The EVM performs poorly in this rather simple scenario and has an AUC of 0.853. This is due to the fact that the unknown class is relatively close to the known class at the bottom compared to the other known classes, even if it is perfectly separated from it. Conversely, the GPDC and the GEVC do not suffer of this type of issues since they do not rely on the distances between the known classes to infer about the unknown ones, and their results in this toy example are close to perfect for both algorithms, with AUC of 0.997 and 0.999, respectively. To show the effectiveness of the GPDC to determine whether new examples are in the support of the training distribution, we report the estimated $\widehat\xi_n$ of the test in point 3.~of the algorithm in Section \ref{GPDCalgorithm} for both known and unknown new data. The right panel of Figure~\ref{fig:1}
shows for this dataset that the $\widehat\xi_n$ estimates for known data are, as expected by Theorem \ref{thm5}, close to $-1/2$. On the other hand, the estimates corresponding to unknown data are much closer to $0$, as it is suggested by Theorem \ref{thm6}. The first test of the GPDC algorithm thus already has a good power to filter out the unknown classes for this data set. We note that, strictly speaking, the supports of all classes are overlapping, but for this finite number of data they are effectively separated so that the test works well.

Figure~\ref{fig:contour} shows the decision boundaries of the GEVC and the GPDC for $\alpha=0.01$ and $\alpha=0.1$. It can be seen that both algorithms produce highly flexible decision boundaries that are capable to follow the shape of the training data.
In some sense, one can see these boundaries as level sets of the training density, and a new point is considered as unknown if it lies outside of these level sets. 
\cite{he2017} also considered extrapolation of level sets into low density regions but using the more restrictive assumption of multivariate regular variation.

\subsection{OLETTER protocol}

\begin{figure}
  \centering
      \includegraphics[scale=0.43]{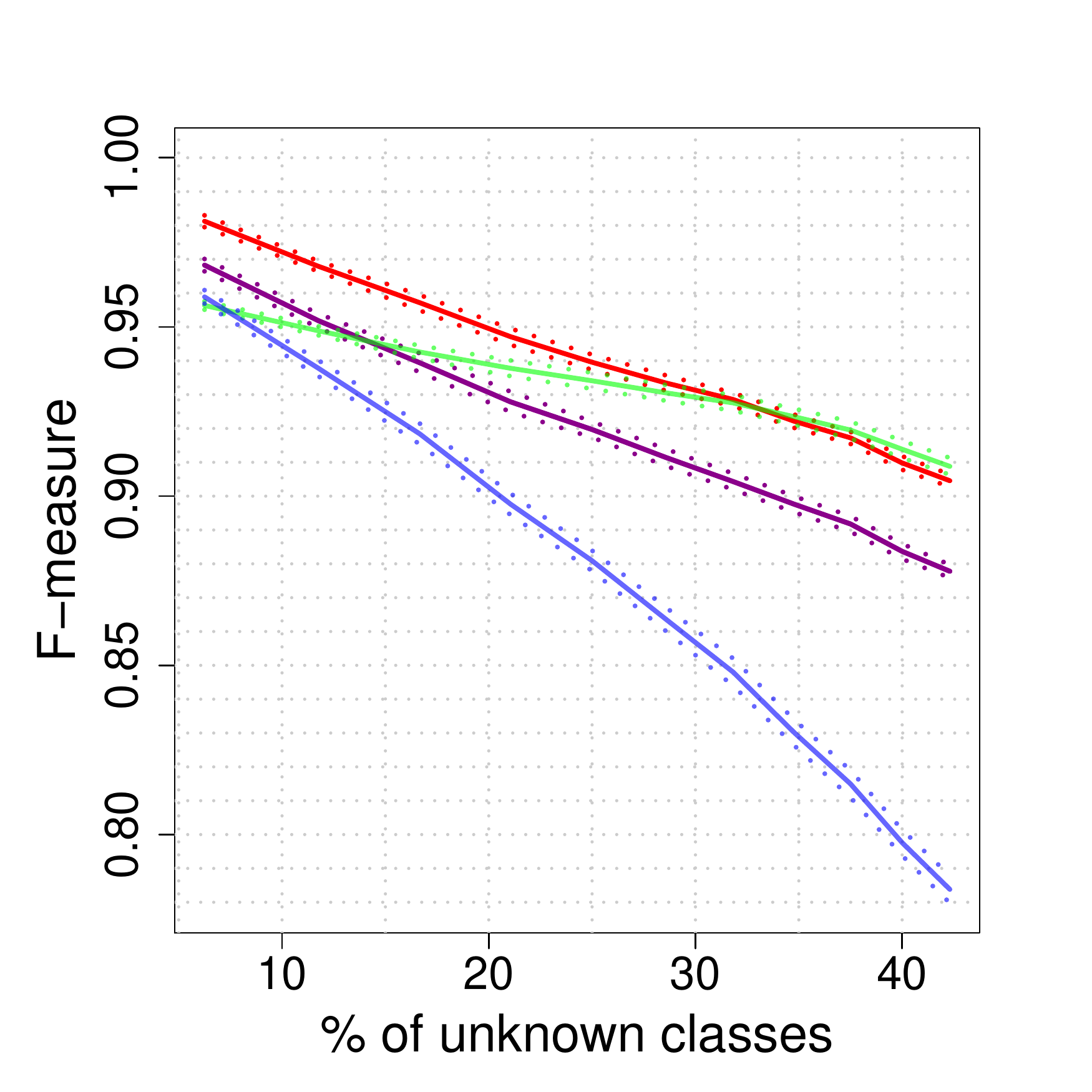}
  \caption{Results for the OLETTER protocol for the EVM (green line), the GPDC (red line), the GEVC (magenta line) and the One-Class SVM (blue line) with one standard deviation confidence intervals (dotted lines). 
}
\label{fig:2}
\end{figure}

In order to evaluate the open set classification performance of the GPDC and the GEVC for real data we compare the two techniques with the EVM and the One-Class SVM using the OLETTER protocol proposed in \citet{bendale2015towards}. This protocol is based on the LETTER data set \citep{frey1991letter} that contains a total of 20000 observations of 26 different classes corresponding to handwritten letters. The predictor space is composed of $p=16$ features that have been extracted from the handwritten letters. The training set has 15000 observations. The protocol consists in randomly selecting 15 classes that are considered as known during training and adding unknown classes by incrementally including subsets of the remaining 11 classes during testing, varying in this way the amount of openness. This process is then repeated 20 times, in a cross-validation fashion. We evaluate the performance of the different algorithms using the obtained $F$-measure \citep{huang2005using} as done in the original paper \citep{bendale2015towards}. For all the four algorithms we set dynamically the probability threshold using the heuristic rule proposed in \citet{rudd2018extreme} that accounts for the amount of openness at each step of the protocol. We set the hyper parameter $k=75$ for the EVM as in original paper, and $k=22$ for the GPDC, roughly corresponding to use the $0.25\%$ of the biggest negated distances. 

The results for the EVM, the GPDC and the GEVC are reported in Figure \ref{fig:2}. It can be seen that all the three methods based on EVT perform better than the One-Class SVM, even if they have fewer hyper parameters than the latter method. Moreover, the hyper parameters of the One-Class SVM were tuned using also unknown examples and that is not realistic during a real use of an open set classifier. The GPDC is the best method in this case, while the GEVC is competitive with the other two even if it has no hyper parameters.

\subsection{Diagnostics of thyroid disease}

\begin{figure*}[htp]
  \centering
  {\includegraphics[scale=0.43]{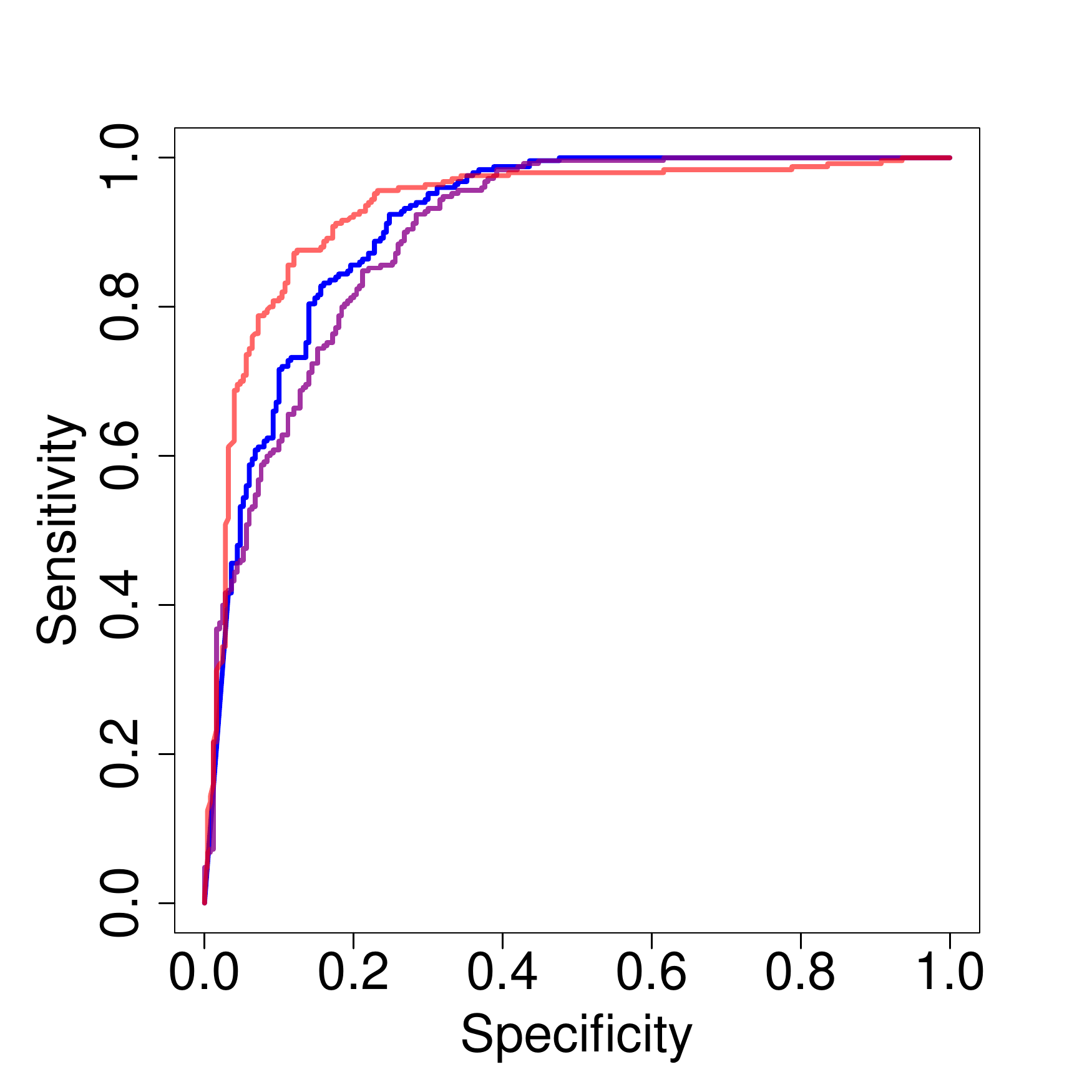}}\quad
  \subfigure{\includegraphics[scale=0.43]{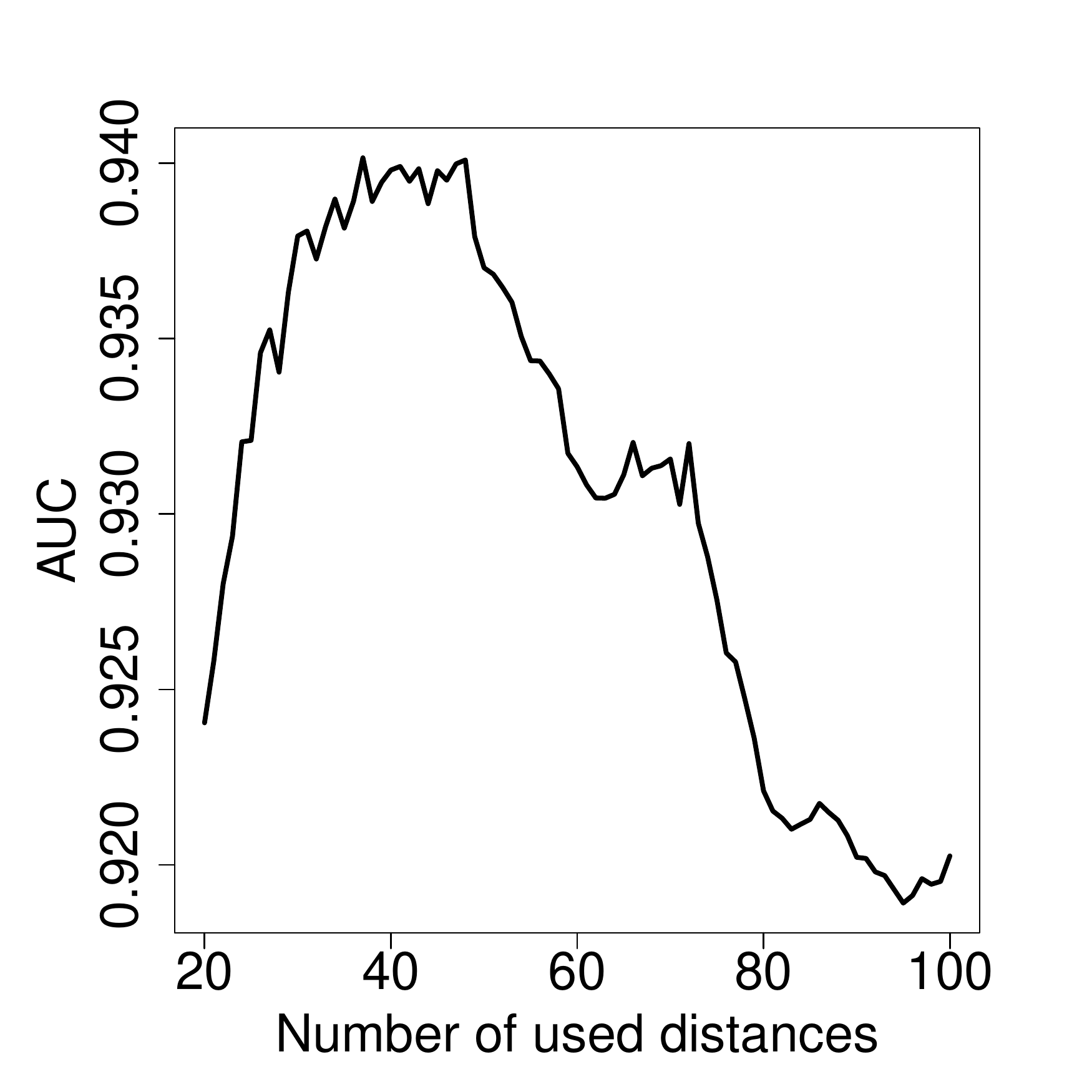}}
\caption{Left panel: ROC curves obtained on the thyroid disease dataset for the One-Class SVM (blue line), the GPDC (red line) and the GEVC (magenta line).
Right panel: AUC as a function of the number $k$ of most extreme used distances for the GPDC.}
\label{fig:3}
\end{figure*}

We consider an application of our algorithms for open set classification to diagnostics of thyroid disease. We analyse the thyroid dataset \citep{quinlan1986inductive,schiffmann1992synthesis}, available at the UCI machine learning repository \citep{lichman2013uci}. The original dataset contains raw clinical measurements from healthy non-hypothyroid and sick patients. These raw measurements were pre-processed in order to easily apply directly neural networks and other common machine learning techniques, resulting in 21 features for each patient. In the final dataset there are 250 sick subjects and 6666 healthy subjects. We consider sick subjects as unknown objects and healthy subjects as known objects. We use all the 250 sick patients together with 250 randomly selected healthy patients to compose the test set. All the remaining healthy patients compose the training set.

Since in this case the known objects belong to only one class, the EVM cannot be used with this dataset. For this reason, we compare the GEVC and the GPDC only with the One-Class SVM. Also here, we use the AUC as evaluation measure. For the GPDC we train it with five different values of $k$, corresponding to using the most extreme $0.25\%$, $1\%$, $2.5\%$, $5\%$ and $10\%$ distances, respectively. These are common choices in extreme value applications and finally we consider the value of $k$ that gives the best ROC curve. For the One-Class SVM with radial kernel we consider different combinations of its hyper parameters and we retain the ones that gives the best results.

The ROC curves obtained for the GPDC, the GEVC and the One-Class SVM are shown in the left panel of Figure \ref{fig:3}. The performance of the three methods are comparable, with an AUC of 0.931, 0.897, 0.912 respectively. We underline again this result since it is particularly favorable to both the GPDC and the GEVC that are reaching the same performance as a state of the art kernel based method like the One-Class SVM. These new methods are kernel free and the GPDC has only one hyper parameter that can be chosen based on EVT, whereas the GEVC has no hyper parameters at all.

The right panel of Figure \ref{fig:3} shows the performance of the GPDC as a function of the number $k$ of most extreme distances used. It can be seen that the GPDC shows a good performance for a wide range of thresholds and the best results are achieved for fairyl small $k$ as suggested wby EVT.

\section{Conclusion}
We present two new kernel free algorithms that perform open set classification using extreme value theory.  These algorithms, called  the GPDC and the GEVC, are fast to update with the arrival of new data and they are easy to adapt to an incremental framework. Moreover, they do not use the geometry of the known classes to infer about the unknowns and are thus able to overcome certain restrictions of previously proposed methods. Their performances are therefore good even when the unknown test data are close to a known class relatively to the other knowns. To show this fact and the effectiveness of the new GPDC and the GEVC, we compare them to the EVM \citep{rudd2018extreme}, another kernel free technique that uses EVT for open set classification, and the more classical kernel based One-Class SVM \citep{scholkopf2000support}, a state of the art technique for novelty detection. The results of our methods on a simulated toy example and on real data sets are competitive. A major strength is that they perform well in very different situations. This good performance in general tasks is probably due to the fact that our methods rely on well-established statistical methods and asymptotically motivated approximations from univariate extreme value theory that apply under very mild conditions. We also underline that both the GPDC and the GEVC are computationally faster than the EVM during the evaluation phase.

The GPDC and the GEVC might be further improved by suitable representations of the data, for example using convolutional neural networks to extract features when working with images and we plan to perform more detailed evaluation of their performance in this direction on standard datasets such as ImageNet \citep{imagenet_cvpr09} or \mbox{CIFAR-100} \citep{krizhevsky2009learning}. Furthermore, to be fully incremental, they need to be capable to store only a subset of the training data. This may be achieved developing a suitable sampling technique that reduces the dimensionality of the training dataset without affecting the asymptotic correctness of the algorithms. 

%Finally, another improvement of the GPDC could be to incorporate an automatic procedure to choose $k$, making it completely free of hyper parameters as the GEVC. This may be achieved starting from known techniques for automatic threshold selection for the GPD distribution \cite{solari2017peaks}.

The result in Theorem \ref{thm5} on the shape parameter of sample distances can be of independent interest. In particular, it will be studied how this can be used for a general method for extrapolation of level sets into low density regions, similarly as in \cite{he2017}.

\section*{Acknowledgments}
Sebastian Engelke was supported by the Swiss National Science Foundation; the paper was completed while he was a visitor at the Department of Statistical Sciences, University of Toronto.

\bibliography{bibliography}

\end{document}